\documentclass[12pt]{article}
\usepackage[utf8]{inputenc}
\usepackage{amsmath, amsfonts, amssymb}
\usepackage{graphicx}
\usepackage{geometry}
\usepackage{pifont}
\usepackage{titlesec}
\usepackage{multirow}
\usepackage{booktabs}
\usepackage{tabularx}
\usepackage{threeparttable}
\usepackage{makecell}
\usepackage{adjustbox}
\usepackage{caption}
\usepackage{enumitem}
\usepackage{xcolor}
\usepackage{tikz}
\usetikzlibrary{arrows.meta, positioning, calc, fit, backgrounds}
\usepackage{hyperref}
\hypersetup{colorlinks=true, linkcolor=blue!55!black, citecolor=green!35!black, urlcolor=blue!55!black}

\usepackage{amsthm}
\newtheorem{definition}{Definition}
\newtheorem{proposition}{Proposition}
\theoremstyle{remark}

\newcommand{\Mem}{\mathcal{M}}                    
\newcommand{\Tool}{\mathcal{T}}                   
\newcommand{\Rew}{\mathcal{R}}                    
\newcommand{\Onto}{\mathcal{O}}                   
\newcommand{\Data}{\mathcal{D}}                   
\newcommand{\Dist}{\mathcal{Q}}                   
\newcommand{\Leak}{\mathcal{P}}                   
\newcommand{\Comm}{\mathcal{C}_{\mathrm{comm}}}   
\DeclareMathOperator{\Update}{Update}
\DeclareMathOperator{\Agg}{Agg}
\DeclareMathOperator{\Adapt}{Adapt}

\titleformat{\section}{\large\bfseries}{\thesection.}{1em}{}
\titleformat{\subsection}{\normalsize\bfseries}{\thesubsection}{1em}{}

\title{\textbf{Federated Agent Optimization}}
\author{
Qiang Yang$^{1,*}$,
Zhiqiang Kou$^{1}$,
Xueyi Zhang$^{3}$,
Dong-Dong Wu$^{4}$,\\
Hanlin Gu$^{5}$,
Jing Guo$^{1}$,
 Yang Liu$^{1}$,
Di Jiang$^{1}$,
Qian Xu$^{2}$
\\[5pt]
\small $^{1}$The Hong Kong Polytechnic University, Hong Kong SAR, China \\
\small $^{2}$Tengen AI, Hong Kong SAR, China \\
\small $^{3}$National University of Singapore, Singapore \\
\small $^{4}$The University of Tokyo, Japan \\
\small $^{5}$Hong Kong University of Science and Technology, Hong Kong SAR, China \\
\small $^{*}$Corresponding author: \texttt{profqiang.yang@polyu.edu.hk}
}

\date{}

\begin{document}
\maketitle

\begin{abstract}
Large language model (LLM) agents increasingly operate in private environments and accumulate valuable experience from task execution, tool use, feedback, and local knowledge. Yet such experience is distributed across organizations and cannot be directly shared because of privacy and proprietary constraints. Conventional federated learning is insufficient for this setting, as agent capabilities extend beyond model parameters to memory, tools, rewards, skills, and structured knowledge.
In this paper, we formulate \textbf{Federated Agent Optimization (FAO)}, which studies how distributed agents can collaboratively improve through controlled information exchange while keeping raw data, complete trajectories, and private knowledge local. We define FAO as a multi-objective problem balancing agent utility, privacy leakage, and communication cost, and organize its optimization space across policy, memory, tool use, reward, and structured knowledge and skills. We further characterize how private experience can be abstracted, protected, aggregated, and adapted into transferable capabilities, providing a unified view of how agents can benefit from one another without direct experience sharing. Finally, we identify the key challenges of FAO and outline several promising directions for future research toward trustworthy federated agent systems.
\end{abstract}
\section{Introduction}

In recent years, large language model (LLM)-driven agents~\cite{guo2024large,naveed2025comprehensive,wang2024survey} have gradually evolved from general-purpose conversational systems~\cite{allen2001architecture} toward executing complex tasks in real-world environments, becoming an important paradigm for enterprise AI systems~\cite{xu2025agentic}. An increasing number of organizations are deploying agents in applications such as customer service~\cite{verhagen2014virtual,gnewuch2017towards,han2023bots}, software development~\cite{qian2024chatdev,bresciani2004tropos}, data analysis~\cite{hong2025data,rasheed2024can}, knowledge management~\cite{van2003towards,wu2001software,roda2003using}, office automation~\cite{wang2024officebench}, financial services~\cite{rizinski2025ai}, healthcare assistance~\cite{preum2021review,milne2020effectiveness,isern2010agents,laranjo2018conversational}, and business process execution~\cite{li2010distributed,jennings2000autonomous,ha2006development,coria2014intelligent}. Unlike traditional language models that primarily focus on content generation~\cite{li2024pre}, enterprise agents~\cite{pan1991intelligent} typically need to integrate internal knowledge~\cite{huang2025reinforced}, business data, proprietary tools, and operational workflows to autonomously perform task planning~\cite{putta2024agent,wang2023describe}, tool invocation~\cite{chen2024re}, environment interaction~\cite{beer1995dynamical,chen2026test}, and outcome evaluation~\cite{zhuge2024agent}. As their applications continue to expand, agents are evolving from standalone language generation models~\cite{dong2022survey} into complex intelligent systems~\cite{zhang2004agent} that can perceive their environments, take actions, and continuously adapt to changing business requirements.

During continuous task execution, enterprise agents generate large volumes of new workflows, execution trajectories~\cite{barke2026agentrx,guo2026se}, tool-use records~\cite{wang2026agent,doshi2026towards}, and success or failure feedback on a daily basis~\cite{ponnusamy2022feedback}. These real-world interaction experiences can help agents continuously refine task planning~\cite{putta2024agent,wang2023describe}, tool utilization, and decision-making strategies~\cite{rizk2018decision}, thereby improving their practical task performance over time. Meanwhile, as agents are deployed across different organizations~\cite{deng2025exploring,zuo2026qwen}, departments, and user environments, valuable experience becomes increasingly distributed across isolated local systems. Different agents often face similar or related tasks, and their accumulated successful experiences, failure patterns, and tool-use strategies are naturally complementary, creating substantial potential for mutual improvement~\cite{li2022pmic}. However, such experience is often tightly coupled with user data, internal knowledge, business rules, proprietary tools, and operational workflows, making it difficult to directly share raw data or complete execution trajectories~\cite{cheng2026privact}. As a result, large amounts of potentially reusable agent experience remain confined to local environments, limiting effective collaboration among agents and preventing the full value of cross-environment experience from being realized.

Federated learning~\cite{wen2023survey,li2020review,zhang2021survey} provides a natural technical foundation for mitigating the experience isolation described above. It enables multiple clients~\cite{geyer2017differentially} to collaboratively learn by exchanging and aggregating model updates while keeping raw data local~\cite{deng2021fair}, thereby avoiding the direct centralization of sensitive information. However, traditional federated learning primarily focuses on model parameters, whereas the capabilities of enterprise agents are determined not only by the underlying model but also by their decision policies, memories, tool use, environmental feedback, and local knowledge and business rules. Therefore, in agent settings~\cite{andriushchenko2025agentharm}, parameter aggregation alone is insufficient to enable effective cross-agent collaboration. A more fundamental question is how to transform the useful experience, skills, and feedback acquired by one agent during real-world execution into shareable capabilities that can benefit other agents without exposing local private information. Existing federated learning paradigms have not yet systematically addressed this problem and still lack a unified optimization framework for the broader capabilities of agents.

We refer to this class of problems as \textbf{Federated Agent Optimization (FAO)}. FAO studies how multiple agents operating in different private environments can collaboratively improve through controlled information exchange while keeping raw data, complete interaction trajectories, and local private knowledge decentralized. Unlike conventional federated learning, which primarily optimizes model parameters, FAO considers the broader capabilities of an agent, including its policy, memory, tool use, reward mechanism, and structured knowledge. We formulate FAO around three interrelated objectives: improving local agent utility, reducing privacy leakage during collaboration, and controlling communication and system costs. In addition, effective federation should provide meaningful collaboration gains to individual participants rather than improving only aggregate performance.

In this paper, we establish a unified framework for studying FAO. First, we formally define local agents, private environments, federated protocols, information-sharing boundaries, and a multi-objective optimization formulation that jointly captures performance, privacy, and communication cost. Second, we organize existing and emerging approaches according to the agent capability being federated, including policies, memory, tool use, rewards, and structured knowledge and skills. Third, we characterize how private local experience can be transformed into transferable capabilities through abstraction, protection, aggregation, and local adaptation, and further examine when such transfer leads to mutual improvement across agents. Building on this framework, we identify key challenges in balancing performance, privacy, and communication, ensuring fair contribution and benefit allocation, and reliably evaluating federation gains. Finally, we highlight several open research directions toward trustworthy, efficient, and continuously improving federated agent systems.

\paragraph{Paper Roadmap.}
The remainder of this paper is organized as follows. Section~2 reviews the most relevant research areas. Section~3 formally defines Federated Agent Optimization (FAO). Section~4 presents a taxonomy of FAO according to the agent components being collaboratively optimized. Section~5 further investigates how local experience can be transformed into transferable capabilities. Section~6 summarizes the key challenges of FAO. Finally, Section~7 discusses  future works.

\vspace{0.5em}
\noindent

\section{Related Work}

\paragraph{Federated Learning and Federated Large Models}
Federated learning~\cite{wen2023survey,li2020review,zhang2021survey} enables multiple clients~\cite{geyer2017differentially} to collaboratively train models while keeping raw data local, providing a general paradigm for privacy-preserving distributed learning~\cite{jayaraman2018distributed}. Early studies~\cite{qi2024model} mainly focused on model aggregation under heterogeneous data distributions, while subsequent work further explored communication efficiency~\cite{chen2021communication}, personalization~\cite{tan2022towards}, robustness~\cite{huang2024federated,li2021ditto}, and privacy protection~\cite{wei2020federated,li2021survey}. With the emergence of foundation models and large language models~\cite{zhao2026survey,minaee2024large}, federated learning has been extended to large-scale model adaptation~\cite{kuang2024federatedscope}, where parameter-efficient techniques such as federated fine-tuning~\cite{Wang2024FLoRA,wu2024fedbiot}, LoRA~\cite{sun2024improving,bian2025lora}, prompt tuning~\cite{zhao2023fedprompt,che2023federated}, and knowledge distillation~\cite{salman2025knowledge} are increasingly adopted to reduce computation and communication costs. Despite these advances, existing federated large-model methods still primarily treat model parameters or representations as the main optimization objects. In contrast, an agent~\cite{guo2024large} typically involves not only an underlying model, but also memory, tool use, rewards, skills, and external knowledge. Therefore, federating a large model is not equivalent to federating an agent, motivating a broader optimization paradigm that considers the complete capabilities of agents.

\paragraph{Multi-Agent Systems and Agent Reinforcement learning}
Multi-agent systems~\cite{he2025llm,yue2025masrouter,li2025agent} study how multiple agents coordinate, communicate, divide tasks, and make joint decisions in shared or interacting environments. With the rise of large language models, recent LLM-based multi-agent systems~\cite{he2025llm,yang2026agentnet,chen2026diversity} have further explored role specialization, discussion, negotiation, and mutual feedback to improve performance on complex tasks. In parallel, agent learning focuses on continuously improving agent capabilities through execution trajectories~\cite{tian2023learning}, environmental feedback~\cite{yuan2023survey}, reinforcement learning~\cite{zhu2024survey,zhang2025reinforcement}, self-reflection~\cite{bo2024reflective}, and teacher-agent guidance~\cite{ye2022model}. Despite this progress, most existing multi-agent and agent-learning methods assume that agents can directly exchange messages, trajectories, or feedback during collaboration, while privacy constraints and data isolation across organizations or private environments are rarely considered. This differs fundamentally from federated agent optimization, where agents are expected to benefit from each other while keeping local data, execution experience, and private knowledge decentralized.

\paragraph{Privacy-Preserving AI} Privacy-preserving AI~\cite{nguyen2025privacy} aims to reduce the risk of sensitive information leakage while enabling effective model training and inference. Common techniques include differential privacy~\cite{alzoubi2025differential}, secure aggregation~\cite{ma2023trusted}, homomorphic encryption~\cite{khan2024privacy}, secure multi-party computation~\cite{saha2024application}, and access control~\cite{ji2026taming}. These techniques have been widely adopted in distributed learning~\cite{chen2021distributed,cao2023communication}, cloud-based inference~\cite{guo2018cloud}, and cross-organization collaboration~\cite{chang2021maintenance,norta2015establishing} to protect training data, model updates, and intermediate representations. However, existing privacy-preserving methods mainly focus on data, model parameters, or inference processes. In agent systems, sensitive information may also be embedded in execution trajectories, memories, tool interactions, business rules, and skill representations, which calls for new protection mechanisms tailored to collaborative agent optimization.

\section{Problem Definition: Federated Agent Optimization}
\label{sec:problem}

We formalize Federated Agent Optimization (FAO) for $K$ clients in the \emph{horizontal} setting, with the notation of Appendix~\ref{app:notation}.
All clients draw their tasks from one task space $\mathcal{X}$, and client $k$ has a task distribution $\Dist_k$, a utility $u_k$, and a private dataset $\Data_k$ of $n_k$ tasks with the trajectories logged by its earlier agents,
\begin{equation}
\label{eq:horizontal}
\operatorname{supp}(\Dist_k) \subseteq \mathcal{X}, \qquad
\Data_k = \big\{(x_{k,i}, \tau_{k,i})\big\}_{i=1}^{n_k}, \quad x_{k,i} \sim \Dist_k, \qquad k = 1, \dots, K.
\end{equation}
The distributions $\Dist_k$ and utilities $u_k$ may differ across clients and coincide in the homogeneous case.
$\Data_k$ also holds the raw records from which the agent builds its memory and knowledge, and none of it may leave the client.

\subsection{Local Agents and Private Data}
\label{sec:local_agent}

Client $k$ starts from an initial agent $A_k^0$, and each of its agents has five components~\cite{Sumers2024CoALA},
\begin{equation}
\label{eq:agent}
A_k = \big(\underbrace{\pi_k}_{\text{policy}},\ \underbrace{\Mem_k}_{\text{memory}},\ \underbrace{\Tool_k}_{\text{tools}},\ \underbrace{\Rew_k}_{\text{reward}},\ \underbrace{\Onto_k}_{\text{structured knowledge}}\big).
\end{equation}
The policy $\pi_k = \pi_{\theta_k}$ maps a context to an action, and $\theta_k$ is continuous (weights, adapters) or discrete (prompts, workflow graphs).
The memory $\Mem_k$ stores experience for retrieval~\cite{Shinn2023Reflexion,Zhao2024ExpeL,Zhang2025MemorySurvey}, $\Tool_k$ is the set of callable tools with their descriptions, and $\Rew_k$ is the local evaluator.
$\Onto_k$ holds curated structured knowledge, namely ontologies, rules on the business and the environment, and skills, which, unlike memory items, are stated explicitly rather than retrieved by similarity.

On a task $x$, an agent $A$ produces a trajectory $\tau$ of $H$ steps, and its performance on client $k$ is the expected utility under the task distribution of the client,
\begin{align}
\tau &= (s_h, a_h, o_h, r_h)_{h=1}^{H} \sim p_{A}(\cdot \mid x), \label{eq:trajectory}\\
U_k(A;\Dist_k) &= \mathbb{E}_{x \sim \Dist_k}\, \mathbb{E}_{\tau \sim p_{A}(\cdot \mid x)} \big[ u_k(\tau) \big],
\qquad u_k(\tau) \in [0,1], \qquad \Rew_k(\tau) \approx u_k(\tau). \label{eq:utility}
\end{align}
Here $s_h$ is the context, $a_h$ the action (a message or a tool call), $o_h$ the observation, and $r_h$ the feedback, which is often given only at the last step.
Since $u_k$ is observed only through outcomes and human feedback, the agent is optimized against the estimate $\Rew_k$.

\subsection{Federated Protocol}
\label{sec:protocol}

A coordinator runs a protocol $\Phi$ that fixes the participants $\mathcal{S}^t$ of each round, the rules $\Update$, $\Agg$, and $\Adapt$, the protection of the uploads, and the schedule of messages.
Round $t = 0, \dots, T-1$ consists of a local update, an aggregation, and a local adaptation,
\begin{align}
\Delta_k^t &= \Update\big(A_k^t, \Data_k; \Phi\big)
= \big(\Delta_k^{\pi},\, \Delta_k^{\Mem},\, \Delta_k^{\Tool},\, \Delta_k^{\Rew},\, \Delta_k^{\Onto}\big)^t, \qquad k \in \mathcal{S}^t, \label{eq:update}\\
G^{t+1} &= \Agg\big(\{\Delta_k^t\}_{k \in \mathcal{S}^t},\, G^t; \Phi\big)
= \big\{G_k^{t+1}\big\}_{k=1}^{K}, \qquad G_k = \big(G_k^{\pi}, G_k^{\Mem}, G_k^{\Tool}, G_k^{\Rew}, G_k^{\Onto}\big), \label{eq:aggregate}\\
A_k^{t+1} &= \Adapt\big(A_k^t, G_k^{t+1}; \Data_k\big), \qquad A_k^{+} = A_k^{T}, \qquad A_k^{\mathrm{loc}} = A_k^{T} \ \text{ under } \ G_k^{t} = \emptyset \ \ \forall t. \label{eq:adapt}
\end{align}
The update is \emph{typed}: it has one block for each component, any block may be empty, and nothing else leaves the client.
Typed artifacts were proposed for tool-routing knowledge~\cite{Chakraborty2026TypedArtifacts}, and Eq.~\eqref{eq:update} extends the idea to all components.
Federated averaging gives all clients one view, $G_k = G$, whereas personalized methods give each client its own view~\cite{Hou2026FedWorld,Yang2026FederatedSkill}; Appendix~\ref{app:extensions} covers protocols without a coordinator.
We drop the index $t$ where possible and write $\Delta_k$ for all uploads of client $k$, $A_k^{+}$ for its final agent, and $A_k^{\mathrm{loc}}$ for the agent that it obtains without the federation (Figure~\ref{fig:fao_overview}).

\begin{figure}[t]
\centering
\begin{tikzpicture}[
  font=\footnotesize,
  chip/.style={draw=blue!45!black, fill=white, rounded corners=1.5pt, minimum width=0.74cm, minimum height=0.48cm, inner sep=0pt},
  up/.style={-{Stealth[length=2.2mm]}, line width=0.9pt, red!65!black},
  down/.style={-{Stealth[length=2.2mm]}, line width=0.9pt, green!40!black},
]
\draw[draw=orange!60!black, fill=orange!10, rounded corners=3pt] (-7.3,1.45) rectangle (7.3,2.7);
\node[align=center] at (0,2.075) {\textbf{Coordinator} with protocol $\Phi$\\[2pt]
  $\{G_k\}_{k=1}^{K} = \Agg\big(\{\Delta_k\}_{k=1}^{K};\,\Phi\big), \qquad G_k = \big(G_k^{\pi}, G_k^{\Mem}, G_k^{\Tool}, G_k^{\Rew}, G_k^{\Onto}\big)$};

\foreach \x/\i in {-5.1/1, 0/k, 5.1/K} {
  \draw[draw=blue!50!black, fill=blue!5, rounded corners=3pt] (\x-2.2,0) rectangle (\x+2.2,-2.3);
  \node at (\x,-0.28) {\textbf{Client $\i$}: agent $A_{\i}$};
  \foreach \s/\dx in {\pi/-1.64, \Mem/-0.82, \Tool/0, \Rew/0.82, \Onto/1.64} {
    \node[chip] at (\x+\dx,-0.84) {$\s_{\i}$};
  }
  \node at (\x,-1.42) {private data $\Data_{\i}$};
  \node at (\x,-1.95) {$A_{\i}^{+}=\Adapt(A_{\i},G_{\i};\Data_{\i})$};
  \draw[up]   (\x-0.55,0)   -- node[left=1pt,  pos=0.25] {$\Delta_{\i}$} (\x-0.55,1.45);
  \draw[down] (\x+0.55,1.45) -- node[right=1pt, pos=0.75] {$G_{\i}$} (\x+0.55,0);
}
\node at (-2.55,-1.15) {$\cdots$};
\node at (2.55,-1.15) {$\cdots$};

\draw[dashed, gray!70!black] (-7.3,0.72) -- (7.3,0.72);
\node[anchor=south east, text=gray!40!black, font=\scriptsize\itshape, inner sep=1.5pt] at (7.3,0.72) {shared};
\node[anchor=north east, text=gray!40!black, font=\scriptsize\itshape, inner sep=1.5pt] at (7.3,0.72) {private};
\end{tikzpicture}
\caption{One round of FAO. Clients upload typed updates $\Delta_k$, the coordinator aggregates them into views $G_k$, and each client adapts its view. Utility is measured at $A_k^{+}$, leakage on $\Delta_k$, and communication cost on every message that crosses the line.}
\label{fig:fao_overview}
\end{figure}
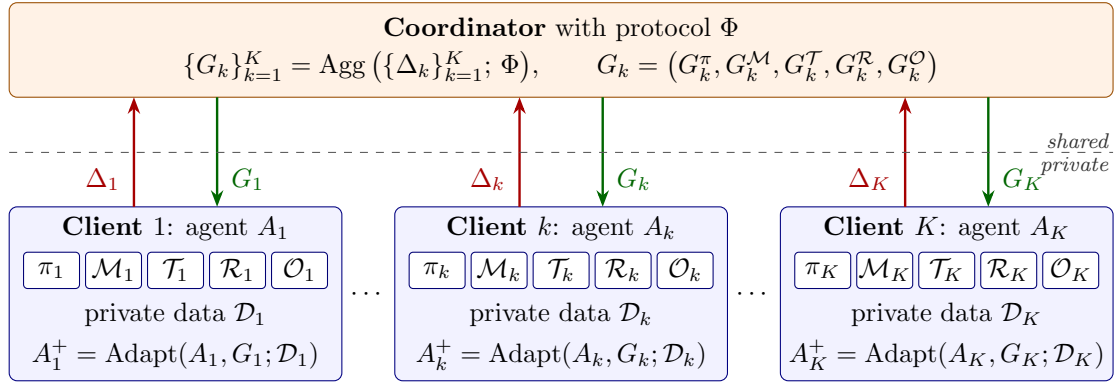

\subsection{Multi-Objective Formulation}
\label{sec:objectives}

FAO pursues three objectives: performance, measured for each client on its own tasks so that personalization needs no separate term, privacy, and communication, whose cost counts every message of the protocol by its size $|\cdot|$,
\begin{equation}
\label{eq:objectives}
\Comm\big(\{\Delta_k\}_{k=1}^{K}, \Phi\big)
= \sum_{t=0}^{T-1} \sum_{k \in \mathcal{S}^t} \Big( \big|\Delta_k^t\big| + \big|G_k^{t+1}\big| \Big).
\end{equation}
Leakage is a property of the randomized mechanism $M_k$ that maps $\Data_k$ to the transcript of uploads, measured by a differential privacy guarantee~\cite{Dwork2014DPFoundations,abadi2016deep} or by the advantage of the strongest attack in a stated threat model $\mathfrak{A}$ over attacks without uploads,
\begin{equation}
\label{eq:leakage}
\begin{aligned}
&\text{transcript:} && M_k(\Data_k) = \big(\Delta_k^{0}, \dots, \Delta_k^{T-1}\big), \\
&\text{guarantee:} && \Leak_k = \inf\big\{\epsilon : \Pr[M_k(\mathcal{D}) \in S] \le e^{\epsilon} \Pr[M_k(\mathcal{D}') \in S] + \delta \ \ \forall S,\ \mathcal{D} \simeq \mathcal{D}'\big\}, \\
&\text{attack:} && \Leak_k = \max_{\mathcal{A} \in \mathfrak{A}} \Pr\big[\mathrm{hit}\big(\mathcal{A}(M_k(\Data_k))\big)\big] - \max_{\mathcal{A} \in \mathfrak{A}} \Pr\big[\mathrm{hit}\big(\mathcal{A}(\bot)\big)\big], \\
&\text{hit:} && \mathrm{hit}(y) \iff \exists\, r \in \Data_k : \mathrm{sim}(y, r) \ge \gamma,
\end{aligned}
\end{equation}
where $\simeq$ is one fixed adjacency relation, such as differing in one trajectory, $\delta$ is fixed, $\bot$ stands for no upload, and a hit recovers a record up to similarity $\gamma$.
The guarantee covers all $T$ rounds and hence their composition~\cite{Dwork2014DPFoundations}, and the advantage measures what the uploads add to one stated threat.
The two notions are not comparable, so an instantiation fixes one scalar metric, written $\Leak_k(M_k)$ or $\Leak_k(\Phi)$, and adds further notions as objectives or constraints.
Gradients~\cite{Zhu2019DLG,Petrov2024DAGER}, embeddings~\cite{Morris2023Vec2Text}, memory~\cite{Wang2025MEXTRA}, and text written by language models~\cite{Shao2024PrivacyLens} can all leak private data.

\begin{definition}[Federated Agent Optimization]
\label{def:fao}
Given $K$ clients with initial agents $\{A_k^0\}$, task distributions $\{\Dist_k\}$, private data $\{\Data_k\}$, and utilities $\{u_k\}$, FAO is the following multi-objective problem over the protocol $\Phi$ of the federation,
\begin{equation}
\label{eq:fao}
\max_{\Phi}\
\Bigg[\
\underbrace{\frac{1}{K}\sum_{k=1}^{K} U_k\big(A_k^{+};\Dist_k\big)}_{\text{local agent performance}},\quad
\underbrace{-\frac{1}{K}\sum_{k=1}^{K} \Leak_k\big(M_k\big)}_{\text{local data leakage risk}},\quad
\underbrace{-\,\Comm\big(\{\Delta_k\}_{k=1}^{K},\Phi\big)}_{\text{system communication cost}}
\ \Bigg],
\end{equation}
where the final agents $A_k^{+}$ and the uploads $\Delta_k$ result from $A_k^0$, $\Data_k$, and $\Phi$ through Eqs.~\eqref{eq:update}--\eqref{eq:adapt}, and $F = (F_1, F_2, F_3)$ denotes the three objectives.
\end{definition}

FAO shares its objectives with constrained multi-objective federated learning~\cite{Kang2024CMOFL} and differs from it in three properties of its decision space.
First, each block of an update is a parameter vector, a text, or a structure such as a graph or a rule set,
\begin{equation}
\label{eq:blocks}
\Delta_k^{c} \ \in\ \underbrace{\mathbb{R}^{d_c}}_{\text{parameters}} \ \cup\ \underbrace{\Sigma^{*}}_{\text{text}} \ \cup\ \underbrace{\mathcal{G}}_{\text{structures}}, \qquad c \in \{\pi, \Mem, \Tool, \Rew, \Onto\},
\end{equation}
where $\Sigma^{*}$ is the set of strings, so $\Agg$ must merge text and structures and not only average.
Second, utility is an expectation over trajectories with sparse feedback, so it has no gradient in the discrete components and only noisy gradient estimates in the continuous ones.
Third, a readable artifact can disclose a record verbatim, whereas a parameter update must first be inverted, so the leakage surface is wider.
FAO is thus not merely federated model training, but collaborative optimization over policies, memories, tools, rewards, skills, and knowledge structures under privacy constraints.

The objectives conflict, as quantified for federated learning of models~\cite{Zhang2022NFL,Zhang2023Tradeoff}: richer updates can raise $F_1$, often at the expense of $F_2$ and $F_3$.
A protocol $\Phi$ with $F(\Phi) = (F_1, F_2, F_3)$ is \emph{Pareto optimal} if no other feasible protocol dominates it,
\begin{equation}
\label{eq:pareto}
\nexists\, \Phi' : \quad F_i(\Phi') \ge F_i(\Phi) \ \ \forall i \in \{1,2,3\} \quad \text{and} \quad F(\Phi') \neq F(\Phi).
\end{equation}
FAO asks for methods that reach the Pareto-optimal set $\mathcal{Z}^{\star}$ and let a practitioner choose a point on the Pareto front $F(\mathcal{Z}^{\star})$ through one of two scalarized forms,
\begin{align}
&\max_{\Phi}\
\frac{1}{K}\sum_{k=1}^{K} U_k\big(A_k^{+};\Dist_k\big)
- \lambda_p\, \frac{1}{K}\sum_{k=1}^{K} \Leak_k\big(M_k\big)
- \lambda_c\, \Comm\big(\{\Delta_k\}_{k=1}^{K},\Phi\big), \label{eq:fao_weighted}\\
&\max_{\Phi}\
\frac{1}{K}\sum_{k=1}^{K} U_k\big(A_k^{+};\Dist_k\big)
\quad \text{s.t.} \quad \Leak_k\big(M_k\big) \le \bar\epsilon_k \ \ \forall k, \quad
\Comm\big(\{\Delta_k\}_{k=1}^{K},\Phi\big) \le B. \label{eq:fao_constrained}
\end{align}
The weights $\lambda_p, \lambda_c > 0$ express soft preferences, and the budgets $\bar\epsilon_k$ and $B$ express hard limits such as a regulatory requirement, which is how enterprises state their needs.
For the exact maximizers $\mathcal{Z}_{\mathrm{w}}$ of Eq.~\eqref{eq:fao_weighted} over all weights and $\mathcal{Z}_{\mathrm{c}}$ of Eq.~\eqref{eq:fao_constrained} over all budgets, Proposition~\ref{prop:inclusion} in Appendix~\ref{app:proofs} and~\cite{Miettinen1999NMOO} give the inclusions
\begin{equation}
\label{eq:inclusion}
\mathcal{Z}_{\mathrm{w}} \subseteq \mathcal{Z}^{\star} \subseteq \mathcal{Z}_{\mathrm{c}}, \qquad
\Phi^{\star} \in \arg\max \eqref{eq:fao_constrained}\big|_{\bar\epsilon_k = \Leak_k(\Phi^{\star}),\ B = \Comm(\Phi^{\star})} \quad \forall\, \Phi^{\star} \in \mathcal{Z}^{\star}.
\end{equation}
Positive weights yield only supported Pareto-optimal protocols, whose images lie on the convex hull of the front, and since the decision space contains text and rules and is not convex, only the constrained form is sure to reach every Pareto-optimal protocol.

A rational client joins the federation only if it gains from it, which adds a participation constraint on the \emph{collaboration gain} of each client,
\begin{equation}
\label{eq:collab_gain}
\mathrm{CG}_k = U_k\big(A_k^{+};\Dist_k\big) - U_k\big(A_k^{\mathrm{loc}};\Dist_k\big) \ \ge\ 0 \qquad \forall k.
\end{equation}
The gain follows from $F_1$ and a baseline, so it is not a fourth objective, and Section~\ref{sec:mutuality} gives a gate that enforces it with high probability, while Appendix~\ref{app:extensions} extends the formulation to inference cost and to task distributions that change over time.

\subsection{Relation to Existing Paradigms}
\label{sec:relation}

Several existing problems are special cases of Eq.~\eqref{eq:fao},
\begin{equation}
\label{eq:special}
\begin{aligned}
&\text{Federated Learning:} && \Delta_k = \big(\theta_k - \theta,\, \emptyset, \emptyset, \emptyset, \emptyset\big), \quad
\theta^{+} = \theta + \textstyle\sum_{k} \frac{n_k}{n}\, \Delta_k^{\pi}; \\
&\text{Experience Pooling:} && \Delta_k = \Data_k, \quad \lambda_p = 0; \\
&\text{Agent Optimization:} && K = 1, \quad \max_{A}\, U(A;\Dist); \\
&\text{Component:} && \Delta_k = \big(\emptyset, \dots, \Delta_k^{c}, \dots, \emptyset\big), \quad \lambda_p = \lambda_c = 0.
\end{aligned}
\end{equation}
They cover federated averaging with $n = \sum_k n_k$~\cite{mcmahan2016communication}, multi-agent systems that pool raw messages or experience without a privacy objective~\cite{Wu2023AutoGen,Li2023CAMEL,Zhang2025GMemory,Tang2025AgentKB}, single-agent optimization~\cite{Khattab2024DSPy,Yuksekgonul2025TextGrad}, and recent methods that federate one component~\cite{Chen2025FedSE,Yao2026FoT,Wu2024FICAL,Yang2026FederatedSkill} with privacy and communication fixed by design.

\section{Taxonomy of Federated Agent Optimization}
\label{sec:taxonomy}

We classify FAO methods by the block of the typed update in Eq.~\eqref{eq:update} that they share (Table~\ref{tab:taxonomy}).
We searched arXiv, the ACL Anthology, OpenReview, and the proceedings of the main machine learning venues up to 29 September 2026 for \emph{federated} combined with \emph{agent}, \emph{memory}, \emph{tool}, \emph{skill}, \emph{reward model}, \emph{prompt}, or \emph{knowledge graph}.
\emph{Agent-level} methods improve a component of a language agent across clients whose data stay local, and \emph{precursors} federate the same kind of component outside an agent, such as the weights of a language model.
The categories overlap (Appendix~\ref{app:taxonomy}).

\begin{table}[!htb]
\centering
\footnotesize
\renewcommand{\arraystretch}{1.25}
\setlength{\tabcolsep}{4pt}
\begin{tabularx}{\textwidth}{@{}>{\raggedright\arraybackslash}p{2.0cm} >{\raggedright\arraybackslash}X >{\raggedright\arraybackslash}X >{\raggedright\arraybackslash}p{2.1cm} >{\raggedright\arraybackslash}p{2.0cm}@{}}
\toprule
\textbf{Category} & \textbf{Shared update} & \textbf{Main leakage channel} & \textbf{Agent-level methods} & \textbf{Precursors} \\
\midrule
Policy ($\Delta^{\pi}$) &
adapter weights or gradients; soft or textual prompts &
inversion of gradients; private examples kept in prompts &
\cite{Wang2025FedMABench,Wang2025MobileA3gent,Wang2026FedGUI,Zhao2026FedWave}, \cite{Chen2025FedSE,Chen2026FedAgent}$^{\ast}$ &
\cite{Zhang2024FedIT,Ye2024OpenFedLLM,Sun2024FedBPT,Chen2025FedTextGrad,Fan2025FedRLHF} \\
Memory ($\Delta^{\Mem}$) &
insights, reasoning abstractions, workflows &
private content quoted in text &
\cite{Yao2026FoT,Li2026FedAgentKE}$^{\ast}$ &
\cite{Wang2024FeB4RAG}, \cite{Addison2024CFedRAG,Mao2025FedE4RAG}$^{\ast}$ \\
Tool use ($\Delta^{\Tool}$) &
tool knowledge as text or typed fields; usage statistics &
schemas of internal systems; arguments of calls &
\cite{Wu2024FICAL,Chakraborty2026TypedArtifacts}$^{\ast}$ &
--- \\
Reward ($\Delta^{\Rew}$) &
parameters of evaluators or preference models; scores &
labels and outcomes of individual cases &
--- &
\cite{Wu2025FedBiscuit,Srewa2025PluralLLM}, \cite{Kim2026FedGD}$^{\ast}$ \\
Structured knowledge ($\Delta^{\Onto}$) &
embeddings of entities and relations; rules; skill patches &
entity sets and triples; business processes &
\cite{Hou2026FedWorld,Yang2026FederatedSkill}$^{\ast}$ &
\cite{Chen2021FedE,Peng2021FKGE,Zhang2022FedR} \\
\bottomrule
\end{tabularx}
\caption{Taxonomy of FAO. Agent-level methods federate a component of a language agent, precursors the same kind of component outside an agent. $^{\ast}$Preprint.}
\label{tab:taxonomy}
\end{table}

\paragraph{Policy ($\Delta^{\pi}$).}
With open weights, clients fine-tune adapters that the coordinator averages with weights $w_k$, and with a closed backbone they share textual prompts $p_k$ that a language model merges, in both cases into one view for all clients,
\begin{equation}
\label{eq:tax_policy}
\begin{alignedat}{3}
&\text{weights:} \quad & \Delta_k^{\pi} &= \theta_k - \theta, \qquad & G^{\pi} &= \textstyle\sum_{k \in \mathcal{S}^t} w_k\, \Delta_k^{\pi}, \qquad \theta_k^{+} = \theta + G^{\pi}; \\
&\text{prompts:} \quad & \Delta_k^{\pi} &= p_k, \qquad & G^{\pi} &= \mathrm{LLM}\big(p_1, \dots, p_K\big), \qquad \theta_k^{+} = G^{\pi}.
\end{alignedat}
\end{equation}
The first form is federated fine-tuning~\cite{Zhang2024FedIT,Ye2024OpenFedLLM,Wu2026FedLLMSurvey}, with variants for heterogeneous adapter ranks~\cite{Wang2024FLoRA,Cho2024HetLoRA,Bai2024FlexLoRA}, differential privacy~\cite{sun2024improving}, and personalization~\cite{Yang2024FedDPA,Lu2024FDLoRA,Guo2025FedSALoRA}.
It has been applied to GUI agents~\cite{Wang2025FedMABench,Wang2025MobileA3gent,Wang2026FedGUI}, to agents tuned on their successful trajectories~\cite{Chen2025FedSE} or by reinforcement learning~\cite{Chen2026FedAgent}, and to workflows that span organizations~\cite{Zhao2026FedWave}.
The second form covers soft and black-box prompts~\cite{Sun2024FedBPT,Lin2023FedBBPT} and textual prompts~\cite{Chen2025FedTextGrad}, and it links FAO to the optimization of prompts and workflows by language models~\cite{yang2023large,Khattab2024DSPy,Yuksekgonul2025TextGrad,Agrawal2026GEPA}.
Weights can be inverted to recover training text~\cite{Zhu2019DLG,Petrov2024DAGER,Sami2025PEFTLeak}, against which secure aggregation and differential privacy are the standard defenses~\cite{Bonawitz2017SecureAgg,abadi2016deep}, whereas prompts are small but readable by every receiver and may quote the private examples from which they were derived.

\paragraph{Memory ($\Delta^{\Mem}$).}
Raw trajectories stay local.
Clients share items $z$ that an abstraction $\alpha$ distills from them, such as insights, workflows, and reasoning strategies~\cite{Shinn2023Reflexion,Zhao2024ExpeL,Wang2025AWM,Ouyang2026ReasoningBank}, the coordinator condenses the items of all clients, and each client retrieves the items $G^{\Mem}(x)$ that match its current task $x$ at inference time,
\begin{equation}
\label{eq:tax_memory}
\Delta_k^{\Mem} = \alpha(\Data_k) = \{z_{k,i}\}_{i}, \qquad
G^{\Mem} = \mathrm{Distill}\big(\textstyle\bigcup_{k} \Delta_k^{\Mem}\big), \qquad
\Mem_k^{+} = \Mem_k \cup G^{\Mem}(x).
\end{equation}
Systems with a trusted store share such items among agents~\cite{Gao2024MemorySharing,Zhang2025GMemory,Tang2025AgentKB,Rezazadeh2025CollaborativeMemory}.
FoT and FedAgentKE avoid centralizing raw trajectories~\cite{Yao2026FoT,Li2026FedAgentKE}, and federated retrieval-augmented generation avoids centralizing documents~\cite{Chakraborty2025FedRAGMapping,Addison2024CFedRAG,Zhao2024FRAG,Mao2025FedE4RAG}.
Items are text, so $\Agg$ is performed by a language model, and the coordinator becomes an agent itself.
Text can quote private content and stored memory can be extracted~\cite{Zeng2024RAGPrivacy,Wang2025MEXTRA,Morris2023Vec2Text}, so FoT audits its uploads by a reconstruction attack, and differentially private generation of text~\cite{Tang2024DPFewShot,Wu2024DPICL,Xie2024AugPE} would add a formal guarantee.

\paragraph{Tool use ($\Delta^{\Tool}$).}
Clients in the horizontal setting use tools of the same kind, and standard interfaces make their descriptions comparable~\cite{Anthropic2024MCP}.
For a tool $v$ and an intent $c$, a client shares its documentation $\mathrm{doc}_k(v)$ and its numbers of calls $N_k$ and successes $N_k^{+}$, from which the coordinator derives a merged documentation and a success rate $\hat q$,
\begin{equation}
\label{eq:tax_tool}
\Delta_k^{\Tool} = \big(\mathrm{doc}_k, N_k, N_k^{+}\big), \qquad
\mathrm{doc}(v) = \mathrm{Merge}\big(\{\mathrm{doc}_k(v)\}_{k}\big), \qquad \hat q(v,c) = \frac{\sum_{k} N_k^{+}(v,c)}{\sum_{k} N_k(v,c)}.
\end{equation}
Only the sums of the counts are needed, so secure aggregation protects them~\cite{Bonawitz2017SecureAgg}.
Descriptions decide how well a tool is used, and they can be refined from experience~\cite{Schick2023Toolformer,Qin2024ToolLLM,Yuan2025EasyTool,Qu2025DRAFT}.
FICAL concatenates the textual tool knowledge of the clients~\cite{Wu2024FICAL}, and typed artifacts merge it field by field across frozen models of different vendors~\cite{Chakraborty2026TypedArtifacts}.
Schemas disclose internal systems, and a shared description is an instruction to every agent that reads it, so $\Agg$ must also resist malicious contributions~\cite{Hou2026MCP}.

\paragraph{Reward ($\Delta^{\Rew}$).}
Clients observe real outcomes but few of them, and the coordinator can run a strong judge, but judges are biased and unreliable on trajectories~\cite{Zheng2023Judging,Lu2025AgentRewardBench}.
A shared evaluator $g_{\phi}$ and a local head $h_{\psi}$, trained jointly with a loss $\ell$, combine both sides, and only the shared part $g_{\phi}$ leaves the client,
\begin{equation}
\label{eq:tax_reward}
(\phi_k, \psi_k) = \arg\min_{\phi, \psi} \sum_{i=1}^{n_k} \ell\big(h_{\psi}(g_{\phi}(\tau_{k,i})),\, u_k(\tau_{k,i})\big), \quad
\Delta_k^{\Rew} = \phi_k - \phi, \quad
\Rew_k^{+} = h_{\psi_k} \circ g_{\phi + G^{\Rew}}.
\end{equation}
Existing methods federate preference or reward models for aligning language models~\cite{Wu2025FedBiscuit,Srewa2025PluralLLM,Kim2026FedGD} and handle heterogeneous $u_k$ by grouping similar clients~\cite{Wu2025FedBiscuit} or by an initialization that adapts quickly~\cite{Kim2026FedGD}, whereas FedRLHF keeps rewards local and federates policies~\cite{Fan2025FedRLHF} and FedEval-LLM uses client models as referees~\cite{He2024FedEvalLLM}.
To our knowledge, none federates an evaluator of agent trajectories, for which process reward models~\cite{Lightman2024Verify,Chae2025WebShepherd} and reward functions written by language models~\cite{Kwon2023RewardDesign,Ma2024Eureka,Xie2024Text2Reward} are building blocks.
A cheaper form shares only the scores of candidate rewards $r^{(j)}$ from the coordinator,
\begin{equation}
\label{eq:tax_reward_scores}
\Delta_k^{\Rew} = \big(\kappa_k^{(j)}\big)_{j=1}^{J}, \qquad
\kappa_k^{(j)} = \operatorname*{corr}_{i \le n_k}\big(r^{(j)}(\tau_{k,i}),\, u_k(\tau_{k,i})\big), \qquad
G^{\Rew} = \arg\max_{r^{(j)}}\, \textstyle\sum_{k} \kappa_k^{(j)}.
\end{equation}
Only $J$ scores leave each client, so leakage and payload stay small.

\paragraph{Structured knowledge and skills ($\Delta^{\Onto}$).}
For an entity $e$ with local embedding $\mathbf{e}_k$, known to the clients $\mathcal{K}_e$, and for a skill library $\mathcal{L}_k$ with local patch $\delta_k$,
\begin{equation}
\label{eq:tax_onto}
G^{\Onto}(e) = \frac{1}{|\mathcal{K}_e|} \sum_{k \in \mathcal{K}_e} \mathbf{e}_k, \qquad
\mathcal{L}_k^{+} = \mathrm{Apply}\Big(\mathcal{L}_k,\ G_k^{\Onto}\Big), \qquad G_k^{\Onto} = \mathrm{Merge}_k\big(\{\delta_j\}_{j=1}^{K}\big).
\end{equation}
FedE averages entity embeddings~\cite{Chen2021FedE}.
They allow triples to be reconstructed, so FedR shares only relation embeddings under secure aggregation~\cite{Zhang2022FedR}, and FKGE and FedLU add differential privacy and personalization~\cite{Peng2021FKGE,Zhu2023FedLU}.
All assume aligned entities, which ontology matching provides, now also with language models~\cite{Euzenat2013OntologyMatching,Giglou2025LLMs4OM,Qiang2024AgentOM}.
FedWorld federates rules on environment dynamics and labels each rule as shared, specific to a cluster, private, or unresolved~\cite{Hou2026FedWorld}.
FederatedSkill exchanges patches to skill libraries, which an agent at the server merges into a personalized view $G_k^{\Onto}$ for each client~\cite{Yang2026FederatedSkill}.
Embeddings reveal which entities a client knows, and skills how it conducts its business.

\paragraph{Comparison.}
The reported payloads differ by roughly three orders of magnitude,
\begin{equation}
\label{eq:payloads}
|\Delta^{\pi}| \approx 76~\text{MB}, \qquad
|\Delta^{\Tool}| \approx 0.08 \text{ to } 20~\text{MB}, \qquad
|\Delta^{\Onto}| \approx 0.1 \text{ to } 0.3~\text{MB},
\end{equation}
for adapters of rank 8 per round~\cite{Chen2025FedSE}, tool knowledge per run or per round~\cite{Wu2024FICAL,Chakraborty2026TypedArtifacts}, and skill patches per run~\cite{Yang2026FederatedSkill}, and parameter updates can also be compressed to kilobytes~\cite{Qin2024FedKSeed}.
The methods that exchange text or rules~\cite{Wu2024FICAL,Yao2026FoT,Li2026FedAgentKE,Yang2026FederatedSkill,Hou2026FedWorld} keep raw trajectories local but apply no formal privacy mechanism, and two of them audit their uploads empirically~\cite{Yao2026FoT,Yang2026FederatedSkill}.
To our knowledge, no method trades the three objectives off across components, which is what Eq.~\eqref{eq:fao} asks for.

\section{How Agents Mutually Improve Each Other}
\label{sec:mutual}

Section~\ref{sec:taxonomy} classified FAO by the component that is optimized.
This section follows the knowledge, from the private experience of one agent to the gain of another, with a focus on semantic updates such as text, rules, and skills.

\subsection{The Improvement Loop}
\label{sec:loop}

Every update maps local evidence, such as trajectories, records, or tool interactions, to an artifact and protects it, $\Update = \rho \circ \alpha$, where $\alpha$ is local training for parametric updates and an abstraction $\alpha_{\ell}$ at a level $\ell$ for semantic ones, which carry knowledge from a sender $j$ to a receiver $k$ in four stages under a protection rule,
\begin{align}
&\Data_j \ \xrightarrow{\ \alpha_{\ell}\ }\ z \ \xrightarrow{\ \rho\ }\ \tilde z \ \xrightarrow{\ \Agg\ }\ G_k \ \xrightarrow{\ \Adapt\ }\ A_k^{+}, \label{eq:loop}\\
&\rho(z) = \mathrm{Sanitize}(z) \ \ \text{if } \big|\{e \in \Data_j : e \models z\}\big| \ge n_{\min}, \qquad \rho(z) = \emptyset \ \ \text{otherwise}. \label{eq:protect}
\end{align}
Here the predicate $e \models z$, evaluated by the sender, states that the evidence item $e$ exhibits the pattern $z$, $n_{\min}$ is a support threshold, and $\mathrm{Sanitize}$ removes identifiers or adds calibrated noise.
Abstraction is a central design dimension of FAO, because knowledge travels as text, rules, and skills, and not only as parameters or as the logits of federated distillation~\cite{Li2019FedMD,Fan2025FedMKT}.
The levels $\boldsymbol{\ell} = (\ell_1, \dots, \ell_K)$ trade the detail of the artifacts against their size, leakage, and reach, and as part of $\Phi$ they are chosen jointly by Eq.~\eqref{eq:fao_weighted},
\begin{equation}
\label{eq:level_choice}
\boldsymbol{\ell}^{\star} \in \arg\max_{\boldsymbol{\ell}}\ F_1(\boldsymbol{\ell}) - \frac{\lambda_p}{K} \sum_{k=1}^{K} \Leak_k(\ell_k) - \lambda_c\, \Comm(\boldsymbol{\ell}).
\end{equation}
A selfish client would instead maximize its own utility minus its own costs and, since its upload mainly helps others, upload as little as possible, the free-riding that incentives counter (Section~\ref{sec:mutuality}).
Higher abstraction should shrink $z$, lower leakage, and widen reach, though not always: an abstract business rule can itself be a trade secret.

\subsection{Five Types of Shared Knowledge}
\label{sec:levels}

Table~\ref{tab:levels} lists five types of shared knowledge with a running example, in which banks verify customer identity with agents and cannot share records.
The types differ in function rather than in degree of abstraction, and they interact: aligned ontologies make experience, skills, and rules easier to ground (Eq.~\ref{eq:conditions}), and the reward type supplies the evaluator that selects what the sender distills.
An artifact of any type may enter the memory, the prompt, or the evaluator of the receiving client (Appendix~\ref{app:levels}).

\begin{table}[!htb]
\centering
\footnotesize
\renewcommand{\arraystretch}{1.25}
\setlength{\tabcolsep}{4pt}
\begin{tabularx}{\textwidth}{@{}>{\raggedright\arraybackslash}p{1.75cm} >{\raggedright\arraybackslash}p{3.5cm} >{\raggedright\arraybackslash}X >{\raggedright\arraybackslash}p{1.8cm}@{}}
\toprule
\textbf{Type} & \textbf{Shared artifact $z$} & \textbf{Example in identity verification} & \textbf{Federated methods} \\
\midrule
Experience & pattern distilled from successes and failures & normalize the transliteration of names before the registry lookup & \cite{Yao2026FoT,Li2026FedAgentKE} \\
Skill & parameterized procedure & recognize the license, query the registry, compare the legal representative, escalate on mismatch & \cite{Yang2026FederatedSkill} \\
Rule & effect of an action in the environment & three failed one-time codes lock the account for 30 minutes & \cite{Hou2026FedWorld} \\
Reward & evaluator or rubric & a dialogue is non-compliant if the agent reveals which field failed the check & \cite{Wu2025FedBiscuit,Kim2026FedGD}$^{\dagger}$ \\
Ontology & alignment of concepts & \emph{beneficial owner} $\equiv$ \emph{ultimate controlling person} & \cite{Chen2021FedE,Peng2021FKGE}$^{\dagger}$ \\
\bottomrule
\end{tabularx}
\caption{Five types of shared knowledge, which differ in function, with a running example. $^{\dagger}$Precursor that federates the component outside an agent (Table~\ref{tab:taxonomy}).}
\label{tab:levels}
\end{table}

\subsection{When Is the Improvement Mutual?}
\label{sec:mutuality}

Let $A_k^{+\setminus j}$ be the agent that client $k$ obtains when the updates of client $j$ are withheld.
The contribution of $j$ to $k$ and the condition for mutual improvement are
\begin{equation}
\label{eq:contribution}
I_{j \to k} = U_k\big(A_k^{+};\Dist_k\big) - U_k\big(A_k^{+\setminus j};\Dist_k\big),
\qquad
I_{j \to k} > 0 \ \ \text{and} \ \ I_{k \to j} > 0.
\end{equation}
Averaged over subsets of clients, such marginal contributions give the Shapley value, which is used to evaluate contributions in federated learning~\cite{Ghorbani2019DataShapley,Wang2020FederatedShapley}.
Three conditions favor the transfer of an item $z$ from $j$ to $k$, with a divergence $\mathrm{div}$ of task distributions, a novelty threshold $\nu$ for known items $y \in \Mem_k \cup \Onto_k$, an alignment $\mu_k$ of shared to local concepts, and the concepts $\mathrm{con}(z)$ and tools $\mathrm{tool}(z)$ that $z$ uses,
\begin{equation}
\label{eq:conditions}
\underbrace{\mathrm{div}(\Dist_j, \Dist_k) \le \mathrm{div}_{\max}}_{\text{relatedness}}, \quad
\underbrace{\max\nolimits_{y} \mathrm{sim}(z, y) < \nu}_{\text{novelty}}, \quad
\underbrace{\mathrm{con}(z) \subseteq \operatorname{dom}(\mu_k), \ \ \mathrm{tool}(z) \subseteq \Tool_k}_{\text{groundability}}.
\end{equation}
They are heuristics, not necessary conditions: distant tasks can share a strategy, a known item can add robustness, and an item can bring a new concept or tool.
Heterogeneity is thus both an obstacle to transfer and a source of its value.
Clustered aggregation addresses relatedness~\cite{Ghosh2020IFCA,Sattler2021ClusteredFL}, and agents trained by federated reinforcement learning tolerate differences in tasks but not in environment dynamics~\cite{Chen2026FedAgent}.

Two mechanisms address these risks.
A gate against negative transfer~\cite{Wang2019NegativeTransfer} adds the items $z_1, \dots, z_m$ one at a time to $A_k^{(0)} = A_k^{\mathrm{loc}}$, where $A \oplus z$ integrates $z$ into $A$, and accepts an item only if its estimated gain exceeds twice an error bound $e_k$,
\begin{equation}
\label{eq:gate}
A_k^{(i)} =
\begin{cases}
A_k^{(i-1)} \oplus z_i & \text{if } \hat U_k\big(A_k^{(i-1)} \oplus z_i\big) - \hat U_k\big(A_k^{(i-1)}\big) \ge 2 e_k, \\
A_k^{(i-1)} & \text{otherwise},
\end{cases}
\qquad i = 1, \dots, m.
\end{equation}
With $A_k^{+} = A_k^{(m)}$ and $\hat U_k$ the mean of $u_k$ on $n_k^{\mathrm{val}}$ held-out tasks from $\Dist_k$, Hoeffding's inequality and a union bound over the $2^m$ reachable agents give
\begin{equation}
\label{eq:gate_bound}
e_k = \sqrt{\frac{(m+1)\ln 2 + \ln(1/\eta)}{2\, n_k^{\mathrm{val}}}}
\quad \Longrightarrow \quad
\Pr\big[\mathrm{CG}_k \ge 0\big] \ \ge\ 1 - \eta,
\end{equation}
so by Proposition~\ref{prop:gate} the gate guarantees individual rationality, that is, the participation constraint~\eqref{eq:collab_gain}, at a margin that grows with $m$ and calls for filtering or batching the items.
Pairwise mutual improvement in the sense of Eq.~\eqref{eq:contribution} is stronger and not implied by Proposition~\ref{prop:gate}, because the items of different clients interact.
Incentives that tie the benefit of a client to its contribution $I_{j \to k}$ encourage it, since a client that only adopts knowledge gains without bearing the privacy cost~\cite{Blum2021OneForOne,Donahue2021ModelSharingGames,Zhan2022IncentiveSurvey,Kou2026WinFLoRA}.
Related safeguards: FedWorld checks the scope of each rule~\cite{Hou2026FedWorld}, Agent KB filters retrieved knowledge by a disagreement gate~\cite{Tang2025AgentKB}, and FederatedSkill builds a library per client~\cite{Yang2026FederatedSkill}.
Each local agent thus remains private and personalized, while the federation lets it benefit from the experience that other agents accumulate in environments it never sees.

\section{Key Challenges}
\label{sec:fao-challenges}

The mechanisms above illustrate how local experience can be transformed into transferable capabilities and ultimately improve other agents. However, they also expose three fundamental challenges that distinguish FAO from conventional federated learning. First, performance, privacy, and communication are inherently coupled rather than independently optimizable objectives. Second, heterogeneous participants may contribute and benefit unevenly, making fairness central to sustained collaboration. Third, the gains of federation are difficult to verify because agent outcomes are long-horizon, environment-dependent, and not yet supported by dedicated evaluation protocols. These challenges extend the statistical and system heterogeneity already studied in federated learning~\cite{karimireddy2020scaffold,lai2021oort} from model optimization to the broader problem of collaborative agent improvement.

\subsection{Performance, Privacy and Communication Trade-offs}
\label{subsec:fao-tradeoff-challenge}

The three objectives of FAO---task performance, privacy protection, and communication efficiency---form a tightly coupled trade-off. Improving performance generally requires sharing information that preserves sufficient detail about local experience, yet richer updates can increase both privacy exposure and communication overhead. Data locality alone does not eliminate this risk: gradients and model updates can reveal information about local training examples~\cite{geiping2020inverting,nasr2019privacy}, while memories, skill descriptions, tool-use knowledge, and structured artifacts may disclose sensitive content more directly. Stronger protection, such as differential privacy, reduces disclosure but can also degrade the information available for optimization~\cite{abadi2016deep}. Conversely, aggressive compression can substantially reduce communication traffic~\cite{lin2018dgc}, but may remove information that is important for capability transfer. Secure aggregation protects individual updates from direct inspection~\cite{Bonawitz2017SecureAgg}, yet this can further constrain validation of what is being shared. Client heterogeneity additionally changes the utility of the same update across participants and can induce objective inconsistency in global optimization~\cite{wang2020fednova}. The central challenge is therefore not to minimize any single cost in isolation, but to identify operating points at which useful capability transfer remains possible under explicit privacy and communication budgets. This makes the Pareto trade-off defined in Section~3 a central systems question for practical FAO.

\subsection{How to Ensure Fairness in Federated Agent Systems?}
\label{subsec:fao-fairness-challenge}

A second challenge is ensuring that the benefits and burdens of federation are distributed fairly across heterogeneous participants. Improvements in federation-wide utility may coexist with substantially different outcomes for individual clients~\cite{li2020fair,mohri2019agnostic}. Some agents may contribute highly valuable experience but receive limited benefit in return, whereas others may benefit primarily from knowledge produced elsewhere. At the same time, globally shared capabilities can induce negative transfer when local objectives, task distributions, tools, or operating constraints differ. Personalized federated learning provides mechanisms for retaining local specialization under shared optimization~\cite{li2021ditto,dinh2020pfedme}, but FAO requires a broader notion of fairness that also accounts for non-parametric artifacts such as memories, skills, rewards, and structured knowledge. Contribution estimation is itself difficult because task outcomes result from long sequences of decisions and feedback may be sparse or delayed~\cite{arjonamedina2019rudder}; attributing recipient-side gains to particular shared contributions introduces an additional layer of credit assignment~\cite{harutyunyan2019hindsight}. The collaboration gain and pairwise contribution defined in Section~5 therefore provide important starting points, but a complete fairness framework must jointly consider what each participant contributes, what benefit it receives, and what privacy, computation, and communication costs it bears. Trustworthy participation ultimately requires that clients retain the ability to personalize, selectively adopt, or reject shared capabilities whose local value does not justify their associated cost or risk.

\subsection{How to Evaluate Federated Agent Systems?}
\label{subsec:fao-verification-challenge}

A third challenge is establishing credible evidence that federated collaboration actually improves agent capability. Agent performance is inherently difficult to verify because success depends on interactions among reasoning, memory retrieval, tool execution, and environmental conditions. Tool-use research shows that apparently similar tools and interfaces can require substantially different execution behavior, while reasoning and environment-facing actions are tightly coupled in interactive agents~\cite{yao2023react}. Consequently, a shared artifact that appears semantically valid may still fail after transfer to a recipient environment. Verification becomes even harder over time: continual optimization can introduce stale capabilities, distribution drift, and interference with previously acquired knowledge~\cite{yoon2021fedweit,zhang2025generalization}. Erroneous or adversarial contributions further complicate the distinction between genuine capability improvement and harmful shared updates~\cite{yin2018byzantine}. Existing benchmarks evaluate general agent competence, web and desktop interaction, or tool--user workflows~\cite{liu2024agentbench,zhou2024webarena,xie2024osworld,yao2025taubench}, while federated benchmarks such as FedScale primarily characterize model and system heterogeneity~\cite{lai2022fedscale}. FAO requires evaluation protocols that combine these perspectives: controlled client heterogeneity, explicit information-sharing boundaries, repeated collaborative optimization, and credible local baselines. Beyond federation-wide task utility, evaluation should report privacy leakage, communication and system costs, client-level collaboration gains, contribution fairness, negative-transfer rates, and robustness under task or participant drift. Without such benchmarks and metrics, it remains difficult to determine when federation truly creates transferable intelligence rather than merely improving an aggregate score.

\section{Future Works}
\label{sec:fao-directions}

A central direction for future FAO research is to build \textbf{\emph{trustworthy federated agent systems}}, where trustworthiness is treated as an end-to-end property of collaborative optimization rather than a single notion of model accuracy or privacy. We highlight several important research questions.

\ding{182} \emph{How can transferred capabilities remain reliable and executable across heterogeneous agent environments?}
A capability that is effective for one agent may rely on specific tools, interfaces, permissions, task distributions, or workflow contexts that do not hold for another agent. Prior work on grounded planning, tool use, and multi-agent workflows highlights the importance of environment-aware execution and capability composition~\cite{huang2022zeroshot,hao2023toolkengpt,li2023apibank,hong2024metagpt}. Future FAO systems should therefore explicitly model the applicability conditions, dependencies, and execution constraints of transferred capabilities, rather than assuming that capabilities remain valid after cross-agent transfer.

\ding{183} \emph{How can federated agent systems protect privacy while remaining secure against unreliable or malicious contributions?}
Keeping raw data local does not eliminate leakage from model parameters, memories, textual abstractions, tool knowledge, or structured representations. Differential privacy and secure aggregation provide useful foundations for limiting such leakage~\cite{abadi2016deep,Bonawitz2017SecureAgg}, while robust aggregation is necessary to tolerate faulty or adversarial participants~\cite{blanchard2017byzantine,cao2021fltrust}. Future research should jointly consider confidentiality, update integrity, and the tension between protecting shared information and validating its quality.

\ding{184} \emph{How can fair contribution and benefit allocation be achieved among heterogeneous agents?}
Different agents may contribute unequal amounts of useful experience, receive different levels of benefit, or even suffer negative transfer. Personalized and fairness-aware federated learning provide useful foundations for preserving local specialization and participant-level utility~\cite{fallah2020personalized,jeong2022factorizedfl,chaudhury2022corestable,lin2020ensemble}. Future FAO research should further quantify contribution, collaboration gain, and participation cost, and ensure that agents receive benefits commensurate with the information, privacy, and resources they contribute.

\ding{185} \emph{How should federated agent systems be evaluated?}
Reliable evaluation should verify whether collaboration genuinely improves local agents compared with credible local baselines, rather than relying only on shared evaluators or aggregate performance. This is particularly challenging when feedback is sparse, trajectories are long, and environments are stochastic~\cite{Christiano2017DeepRLHF,ouyang2022instructgpt,rafailov2023dpo,madaan2023selfrefine}. Existing agent and federated benchmarks provide useful starting points~\cite{liu2024agentbench,zhou2024webarena,xie2024osworld,yao2025taubench,lai2022fedscale}, but FAO additionally requires explicit information-sharing boundaries, heterogeneous participants, repeated collaboration, and longitudinal evaluation of recipient-level gains, negative transfer, and robustness under drift.

\ding{186} \emph{How should the full lifecycle cost of federated agent collaboration be measured and optimized?}
Communication volume alone is insufficient to characterize the cost of FAO. Collaboration may also incur substantial costs for local adaptation, privacy protection, validation, inference, tool execution, storage, synchronization, and continual maintenance, especially under heterogeneous computation and model capacities~\cite{alam2022fedrolex,mei2022flanc,diao2021heterofl}. Future work should therefore develop lifecycle-level cost models that jointly account for the resources required to produce, protect, aggregate, validate, deploy, and maintain shared capabilities.

Taken together, these questions suggest that trustworthy FAO should jointly consider execution reliability, privacy and security, fairness, credible evaluation, and lifecycle cost, providing a principled basis for collaborative learning among autonomous agents across organizations.
\section{Conclusion}
\label{sec:fao-conclusion}

This paper formulates Federated Agent Optimization (FAO) as a general framework for collaborative improvement across policies, memories, tool use, rewards, and structured knowledge under distributed data and local constraints. By extending federated optimization beyond model parameters, we establish a multi-objective problem formulation, a taxonomy of federated agent components, and an improvement framework that connects local experience to transferable capabilities and recipient-side gains. Our analysis highlights three central challenges---balancing performance, privacy, and communication cost, ensuring fair contribution and benefit allocation, and reliably verifying the gains of federation---which together motivate the broader goal of building trustworthy FAO systems with reliable execution, privacy and security protection, credible evaluation, and principled cost accounting. We envision FAO as a foundation for scalable, trustworthy, and continuously evolving enterprise AI systems that benefit from collective experience while preserving local specialization and organizational control.

\bibliographystyle{plain}
\bibliography{ref}

\appendix

\section{Notation}
\label{app:notation}

Table~\ref{tab:notation} summarizes the notation of Sections~\ref{sec:problem}--\ref{sec:mutual}.

\begin{table}[!htb]
\centering
\small
\renewcommand{\arraystretch}{1.15}
\begin{tabularx}{\textwidth}{@{}l X@{}}
\toprule
\textbf{Symbol} & \textbf{Meaning} \\
\midrule
$K$, $k$, $\mathcal{X}$ & number and index of clients; task space shared by all clients \\
$\Dist_k$, $u_k(\tau)$ & task distribution of client $k$; utility of a trajectory \\
$\Data_k$, $n_k$ & private dataset of client $k$ (tasks, logged trajectories, raw records); its number of tasks \\
$A_k^0$, $A_k^t$ & initial agent of client $k$; its agent after round $t$ \\
$A_k = (\pi_k, \Mem_k, \Tool_k, \Rew_k, \Onto_k)$ & policy, memory, tools, reward, structured knowledge (ontologies, rules, skills) \\
$\theta_k$ & parameters of the policy $\pi_k$, continuous or discrete \\
$\tau$, $H$ & trajectory and its number of steps \\
$U_k(A;\Dist_k)$ & expected utility of agent $A$ on the tasks of client $k$ \\
$\Phi$ & protocol: participants, rules $\Update$, $\Agg$, $\Adapt$, protection, schedule \\
$\mathcal{S}^t$, $T$ & clients that participate in round $t$; number of rounds \\
$\Delta_k = (\Delta_k^{\pi}, \Delta_k^{\Mem}, \Delta_k^{\Tool}, \Delta_k^{\Rew}, \Delta_k^{\Onto})$ & typed update of client $k$, each block in $\mathbb{R}^{d_c} \cup \Sigma^{*} \cup \mathcal{G}$ \\
$G = \{G_k\}_{k=1}^{K}$ & aggregated state with the view $G_k$ of client $k$ \\
$A_k^{+}$, $A_k^{\mathrm{loc}}$, $A_k^{+\setminus j}$ & agent of client $k$ after the federation; without the federation; without the updates of client $j$ \\
$F = (F_1, F_2, F_3)$ & objectives: performance, privacy, communication \\
$M_k$ & mechanism that maps $\Data_k$ to the transcript of uploads of client $k$ \\
$(\epsilon, \delta)$; $\mathfrak{A}$, $\mathrm{sim}$, $\gamma$ & differential privacy parameters; threat model, similarity, and threshold of an attack \\
$\Leak_k(M_k)$, $\Comm(\{\Delta_k\},\Phi)$ & privacy leakage of client $k$, also written $\Leak_k(\Phi)$; communication cost of the protocol \\
$\simeq$, $\bot$, $\mathrm{hit}$ & adjacency of datasets; no upload; success of an attack \\
$\lambda_p$, $\lambda_c$; $\bar\epsilon_k$, $B$ & weights of the privacy and communication objectives; their budgets \\
$\mathcal{Z}^{\star}$, $\mathcal{Z}_{\mathrm{w}}$, $\mathcal{Z}_{\mathrm{c}}$ & Pareto-optimal protocols; maximizers of the weighted form; of the constrained form \\
$\mathrm{CG}_k$, $I_{j \to k}$ & collaboration gain of client $k$; contribution of client $j$ to client $k$ \\
$\alpha$, $\alpha_{\ell}$, $\rho$, $z$ & map from local evidence to an artifact; abstraction at level $\ell$; protection; artifact \\
$e \models z$, $n_{\min}$ & evidence item $e$ exhibits pattern $z$; support threshold \\
$\boldsymbol{\ell} = (\ell_1, \dots, \ell_K)$ & levels of abstraction of all clients, part of $\Phi$ \\
$g_{\phi}$, $h_{\psi}$ & shared evaluator and local head of the reward model \\
$N_k$, $N_k^{+}$ & numbers of calls and of successful calls of a tool \\
$\mathrm{div}$, $\nu$, $\mu_k$ & divergence of task distributions; novelty threshold; alignment of concepts \\
$A \oplus z$, $\hat U_k$ & agent $A$ after integrating item $z$; utility estimated on validation tasks \\
$m$, $n_k^{\mathrm{val}}$, $e_k$, $\eta$ & number of items; number of validation tasks; error bound; failure probability \\
\bottomrule
\end{tabularx}
\caption{Notation.}
\label{tab:notation}
\end{table}

\section{Methods that Federate a Component of an Agent}
\label{app:taxonomy}

This appendix extends Section~\ref{sec:taxonomy}.
Table~\ref{tab:methods} lists the agent-level methods and selected precursors, and the following subsections describe each category together with the work on single agents on which it builds.

\subsection{Federated Policy Optimization}
\label{app:tax_policy}

The policy block $\Delta_k^{\pi}$ carries parameters of the policy.
Its form depends on how much access the clients have to the backbone.

\paragraph{Weights.}
When the backbone is open, clients fine-tune it on local data with parameter-efficient methods and the coordinator aggregates the updates~\cite{Zhang2024FedIT,Ye2024OpenFedLLM,Zhang2023FedPETuning,Wu2026FedLLMSurvey}.
Later work addresses clients with different adapter ranks~\cite{Wang2024FLoRA,Cho2024HetLoRA,Bai2024FlexLoRA} and the instability of adapters under differential privacy noise~\cite{sun2024improving}.
Personalized variants keep part of the adapter local~\cite{Yang2024FedDPA,Lu2024FDLoRA,Guo2025FedSALoRA}.
At the other extreme of communication cost, zeroth-order optimization lets clients and server exchange only random seeds and scalar gradients~\cite{Qin2024FedKSeed}.

The same recipe has been applied to agents.
GUI agents based on vision-language models are trained on decentralized user data~\cite{Wang2025FedMABench,Wang2025MobileA3gent,Wang2026FedGUI}.
Fed-SE fine-tunes adapters on the successful trajectories of each client~\cite{Chen2025FedSE}.
FedAgent trains the policy by reinforcement learning on local trajectories and averages the parameters~\cite{Chen2026FedAgent}.
It finds the result robust to heterogeneity of tasks but not to heterogeneity of environment dynamics.
This agrees with the theory of federated reinforcement learning~\cite{Qi2021FRLSurvey,Fan2021FedPGBR,Khodadadian2022FedRLSpeedup}, in which the suboptimality of the learned policy grows with the heterogeneity of the environments~\cite{Jin2022FedRLHetero}.
FedWave federates the role-specific adapters of a workflow whose stages belong to different organizations~\cite{Zhao2026FedWave}.
Frameworks that separate the execution of an agent from its training make such methods applicable to existing agents~\cite{Luo2025AgentLightning}.

\paragraph{Prompts.}
When the backbone is a frozen model behind an API, the policy is optimized through its prompts.
Clients can share continuous prompts~\cite{Guo2024PromptFL}, optimize prompts without gradients by querying the model~\cite{Sun2024FedBPT,Lin2023FedBBPT}, or share textual prompts that the coordinator merges by summarization~\cite{Chen2025FedTextGrad}.
The last option connects FAO to the optimization of prompts and agentic workflows~\cite{yang2023large,Khattab2024DSPy,Yuksekgonul2025TextGrad,Agrawal2026GEPA}, where the variables are text and the optimizer is often itself a language model.

\paragraph{Trade-offs.}
Weight sharing is the most mature form and the most expensive in communication: Fed-SE reports 76.3~MB per round for adapters of rank 8~\cite{Chen2025FedSE}.
Its leakage is well studied.
Gradients can be inverted to recover training text, for recent attacks exactly and for large batches~\cite{Zhu2019DLG,Gupta2022FILM,Petrov2024DAGER}, and restricting the update to adapters does not remove the risk~\cite{Sami2025PEFTLeak}.
Secure aggregation and differentially private training are the standard defenses~\cite{Bonawitz2017SecureAgg,abadi2016deep}.
Prompt sharing has a small payload and works with closed backbones, but a textual prompt is readable by anyone who receives it and may quote the examples from which it was derived.

\begin{table}[!htb]
\centering
\footnotesize
\renewcommand{\arraystretch}{1.25}
\setlength{\tabcolsep}{4pt}
\begin{tabularx}{\textwidth}{@{}>{\raggedright\arraybackslash}p{3.3cm} p{1.0cm} >{\raggedright\arraybackslash}X >{\raggedright\arraybackslash}X@{}}
\toprule
\textbf{Method} & \textbf{Block} & \textbf{Shared update} & \textbf{Aggregation} \\
\midrule
\multicolumn{4}{@{}l}{\emph{Agent-level methods}} \\
FedMABench~\cite{Wang2025FedMABench}, FedGUI~\cite{Wang2026FedGUI} & $\Delta^{\pi}$ & adapters of GUI agents based on vision-language models & standard federated algorithms \\
MobileA3gent~\cite{Wang2025MobileA3gent} & $\Delta^{\pi}$ & models trained on automatically annotated phone usage & averaging weighted by episode and step statistics \\
Fed-SE$^{\ast}$~\cite{Chen2025FedSE} & $\Delta^{\pi}$ & adapters tuned on successful trajectories & unweighted averaging of adapters \\
FedAgent$^{\ast}$~\cite{Chen2026FedAgent} & $\Delta^{\pi}$ & parameters trained by reinforcement learning & federated averaging \\
FedWave~\cite{Zhao2026FedWave} & $\Delta^{\pi}$ & role-specific adapters of a multi-stage workflow & federated adapters fused by a shared mixture-of-experts router \\
FoT$^{\ast}$~\cite{Yao2026FoT} & $\Delta^{\Mem}$ & summaries of reasoning traces & clustering and distillation into an insight library \\
FedAgentKE$^{\ast}$~\cite{Li2026FedAgentKE} & $\Delta^{\Mem}$ & knowledge units with workflows, corrections, and utility scores & clustering, selection, deduplication \\
FICAL$^{\ast}$~\cite{Wu2024FICAL} & $\Delta^{\Tool}$ & textual knowledge on tool use & concatenation \\
Typed artifacts$^{\ast}$~\cite{Chakraborty2026TypedArtifacts} & $\Delta^{\Tool}$ & schema-validated tool-routing knowledge & field-wise merging \\
FedWorld$^{\ast}$~\cite{Hou2026FedWorld} & $\Delta^{\Onto}$ & abstract transition rules & alignment of rules; scope inferred from supporting and contradicting evidence \\
FederatedSkill$^{\ast}$~\cite{Yang2026FederatedSkill} & $\Delta^{\Onto}$ & patches to the local skill library & merging and rewriting by a server-side agent, separately for each client \\
\midrule
\multicolumn{4}{@{}l}{\emph{Selected precursors}} \\
FedTextGrad~\cite{Chen2025FedTextGrad} & $\Delta^{\pi}$ & textual prompts of a language model & concatenation or summarization by a language model \\
FedRLHF~\cite{Fan2025FedRLHF} & $\Delta^{\pi}$ & policies trained on local rewards from human feedback, which stay local & federated aggregation of the policies \\
FedE4RAG$^{\ast}$~\cite{Mao2025FedE4RAG} & $\Delta^{\Mem}$ & encrypted parameters of retrievers & averaging with distillation between server and clients \\
FedBiscuit~\cite{Wu2025FedBiscuit} & $\Delta^{\Rew}$ & binary selectors that encode client preferences & aggregation within groups of similar clients \\
FedGD$^{\ast}$~\cite{Kim2026FedGD} & $\Delta^{\Rew}$ & parameters of a reward model & averaging with group-debiased client sampling \\
FedR~\cite{Zhang2022FedR} & $\Delta^{\Onto}$ & relation embeddings of a knowledge graph & secure aggregation over aligned relations \\
\bottomrule
\end{tabularx}
\caption{Agent-level FAO methods and selected precursors, as instances of the typed update of Eq.~\eqref{eq:update}. $^{\ast}$Preprint.}
\label{tab:methods}
\end{table}

\subsection{Federated Memory Optimization}
\label{app:tax_memory}

The memory block $\Delta_k^{\Mem}$ carries experience in abstracted form.
Raw trajectories are the most informative content of memory and the most sensitive, so they stay local.
What can be shared are the products that agents already derive from trajectories for their own use: reflections on failures~\cite{Shinn2023Reflexion}, insights extracted from successes and failures~\cite{Zhao2024ExpeL}, workflows induced from repeated tasks~\cite{Wang2025AWM}, and reasoning strategies~\cite{Ouyang2026ReasoningBank}.

Several systems let agents share such items through a common store: a pool of prompt--answer pairs~\cite{Gao2024MemorySharing}, a hierarchical memory for the agents of one system~\cite{Zhang2025GMemory}, and a knowledge base of workflows and fixes built from the trajectories of different agent frameworks~\cite{Tang2025AgentKB}.
Access control over memory that is shared among users has also been studied~\cite{Rezazadeh2025CollaborativeMemory}.
These systems assume that the store is trusted.

Federated methods avoid centralizing raw trajectories in such a store.
In FoT, each client uploads summaries of its reasoning traces, without the problem instances or task instructions, and the server distills them into a library of insights that all clients can use~\cite{Yao2026FoT}.
FedAgentKE exchanges knowledge units that contain reusable workflows and corrections together with a utility score, and each client retrieves the units relevant to its task~\cite{Li2026FedAgentKE}.
For document collections, federated retrieval-augmented generation~\cite{Chakraborty2025FedRAGMapping} answers queries over the collections of several providers without centralizing the documents, protected by confidential computing~\cite{Addison2024CFedRAG} or homomorphic encryption~\cite{Zhao2024FRAG}, and it trains the retrievers of the clients by federated learning~\cite{Mao2025FedE4RAG}.

\paragraph{Aggregation.}
Memory items are text, so they cannot be averaged.
The coordinator clusters the items, removes duplicates, merges items that express the same pattern, and ranks the result by utility or by the number of clients that support it.
A language model performs these operations, which makes the coordinator an agent in its own right.

\paragraph{Trade-offs.}
The payload is small: in one setting of FoT, the traces of all clients amount to 33K tokens and the distilled library to 3K tokens~\cite{Yao2026FoT}.
The leakage is of a different kind than for weights.
An item written in natural language may quote private content, retrieval stores and agent memories can be extracted by crafted queries~\cite{Zeng2024RAGPrivacy,Wang2025MEXTRA}, and embeddings can be inverted~\cite{Morris2023Vec2Text}.
Protection therefore has to act on the text.
FoT audits its uploads with a reconstruction attack~\cite{Yao2026FoT}.
Differentially private generation of demonstrations~\cite{Tang2024DPFewShot,Wu2024DPICL} and of synthetic text~\cite{Xie2024AugPE} would add a formal guarantee, but to our knowledge it has not yet been applied to the memory of agents.

\subsection{Federated Tool-Use Optimization}
\label{app:tax_tool}

The tool block $\Delta_k^{\Tool}$ carries knowledge about how to use tools.
In the horizontal setting the clients use tools of the same kind, such as customer databases, payment interfaces, or document recognition, and standard interfaces such as the Model Context Protocol~\cite{Anthropic2024MCP} make the descriptions of tools comparable across clients.

Research on single agents shows what this knowledge consists of~\cite{Qin2024ToolLearning,Qu2025ToolSurvey}.
Models learn to call tools from examples of correct calls~\cite{Schick2023Toolformer,Patil2024Gorilla,Qin2024ToolLLM}.
Concise and consistent descriptions improve the use of a tool~\cite{Yuan2025EasyTool}, and a description can be refined from the experience of calling the tool~\cite{Qu2025DRAFT}.
Agents can also create new tools~\cite{Cai2024LATM,Qian2023CREATOR,Wang2024TroVE}.

Two methods federate such knowledge.
In FICAL, each client writes a compendium that describes its tools, the scenarios in which to use them, precautions, and the coordination of several tools; the server concatenates the compendiums, and the agents retrieve from the result~\cite{Wu2024FICAL}.
Typed federated artifacts replace free text by schema-validated objects with separate fields for privacy handling, conflict resolution, and transfer across models, and they are designed for frozen models of different vendors~\cite{Chakraborty2026TypedArtifacts}.
Usage statistics, such as the success rate of a tool for a given intent, are a third kind of tool knowledge.
Only their sum is needed, so they can be protected by secure aggregation.

\paragraph{Trade-offs.}
Tool knowledge in text is small compared with model weights: FICAL reports 0.078~MB against 26{,}000~MB for baselines that exchange parameters~\cite{Wu2024FICAL}.
It is not small by nature, however, and a compendium of tool-routing experience can reach 20~MB per client and round~\cite{Chakraborty2026TypedArtifacts}.
The leakage concerns the systems of the client more than its data: tool schemas disclose the structure of internal systems, and the arguments of recorded calls may contain customer records.
Shared tool knowledge is also an attack surface.
A tool description is an instruction to the agent that reads it, and threats from malicious tool providers are documented for the Model Context Protocol~\cite{Hou2026MCP}.
Aggregation of tool knowledge therefore needs robustness against malicious contributions in addition to privacy.
Evaluation is a further open issue: the study of typed artifacts reports that a simple retrieval baseline outperforms all of its routing variants based on language models, and concludes that the benchmark cannot measure the benefit of shared routing knowledge~\cite{Chakraborty2026TypedArtifacts}.

\subsection{Federated Reward Optimization}
\label{app:tax_reward}

The reward block $\Delta_k^{\Rew}$ carries knowledge about how to judge a trajectory.
The problem has a natural division of labor.
Clients observe real outcomes and human judgments, but each client has few of them.
The coordinator can run a strong evaluator, but language models used as judges show systematic biases~\cite{Zheng2023Judging}, and on the trajectories of web agents no single judge is reliable across benchmarks~\cite{Lu2025AgentRewardBench}.

\paragraph{Federated preference and reward models.}
Existing methods address the alignment of language models with human preferences~\cite{Christiano2017DeepRLHF}.
In FedRLHF, the feedback and the reward function stay with the client, and the policies are federated~\cite{Fan2025FedRLHF}.
FedBis and FedBiscuit federate the preference model itself: each client encodes its preferences in a binary selector, and the selectors are aggregated~\cite{Wu2025FedBiscuit}.
Reward models and preference predictors are also trained by federated averaging and then personalized by local fine-tuning~\cite{Srewa2025PluralLLM,Kim2026FedGD}, and a federation of client models can serve as a panel of referees~\cite{He2024FedEvalLLM}.
To our knowledge, no method yet federates an evaluator of agent trajectories.

\paragraph{Building blocks for agents.}
Three lines of work on single agents provide the pieces.
Evaluators can be trained or prompted to judge the success of web navigation and device control~\cite{Pan2024AutonomousEvaluation}.
Since feedback arrives at the end of a trajectory, credit must be assigned to individual steps~\cite{Pignatelli2024CreditAssignment}; process reward models, which score each step, serve this purpose~\cite{Lightman2024Verify,Chae2025WebShepherd}.
Language models can write reward functions from a description of the task~\cite{Kwon2023RewardDesign,Ma2024Eureka,Xie2024Text2Reward}.
The last line suggests a form of federation with very low cost: a language model at the coordinator proposes a reward function or a rubric, each client scores the proposal against its private outcomes, and only the scores return to the coordinator.

\paragraph{Trade-offs.}
The utilities $u_k$ differ across clients, so a single global evaluator may be wrong for every client.
The methods above respond by grouping clients with similar preferences~\cite{Wu2025FedBiscuit} or by learning an initialization that adapts quickly to each client~\cite{Kim2026FedGD}.
The main leakage channel is the label: a preference or an outcome label reveals how an individual case ended.

\subsection{Federated Structured Knowledge and Skill Optimization}
\label{app:tax_onto}

The block $\Delta_k^{\Onto}$ carries structured knowledge: ontologies and knowledge graphs, rules on the business and the environment, and skills.

\paragraph{Ontologies.}
Federated learning of knowledge graph embeddings lets clients with overlapping entities improve their embeddings without exchanging triples.
In FedE, clients upload entity embeddings and the server averages them~\cite{Chen2021FedE}.
Entity embeddings, however, allow private triples to be reconstructed, and FedR therefore aggregates only relation embeddings, under secure aggregation, which lowers both leakage and communication cost~\cite{Zhang2022FedR}.
FKGE needs no server: pairs of clients translate the embeddings of aligned entities and relations into each other's space through an adversarial network with a differential privacy guarantee~\cite{Peng2021FKGE}.
FedLU keeps the local embeddings apart from the global ones and transfers knowledge between them by distillation~\cite{Zhu2023FedLU}.

These methods assume that the entities and relations of the clients have been aligned beforehand.
Ontology matching is a classical problem~\cite{Euzenat2013OntologyMatching}.
Language models and agents now reach the level of established matching systems on standard benchmarks and improve on them in complex cases~\cite{Hertling2023OLaLa,Giglou2025LLMs4OM,Qiang2024AgentOM}, and the combination of ontologies and language models has been used to harmonize data before federated learning~\cite{Kokash2026OntologyLLMHarmonization}.

\paragraph{Rules.}
FedWorld federates what agents learn about the dynamics of their environment.
Each client converts its private transitions into abstract rules.
The server aligns the rules of different clients, collects the evidence that supports and contradicts each rule, and labels the rule as shared, specific to a cluster of clients, private, or unresolved~\cite{Hou2026FedWorld}.

\paragraph{Skills.}
A skill packages a procedure so that an agent can load it when needed~\cite{Wang2024Voyager,Zheng2025SkillWeaver,Anthropic2025AgentSkills}.
In FederatedSkill, clients upload patches to their skill libraries, from which task-specific values are omitted, and an agent at the server merges and rewrites the patches into a library for each client~\cite{Yang2026FederatedSkill}.

\paragraph{Trade-offs.}
An aligned ontology helps the other categories, since insights, tool descriptions, and rubrics all refer to business concepts.
It is also expensive to build and changes slowly, so its communication cost is incurred mostly once.
Skills are cheap to exchange: FederatedSkill reports between 99 and 281~KB for a family of tasks over a complete run~\cite{Yang2026FederatedSkill}.
The leakage concerns existence and process: the entity embeddings of a client reveal which entities it knows and can expose the relations among them~\cite{Zhang2022FedR}, and a skill reveals how a client conducts its business.

\subsection{Comparison}
\label{app:tax_compare}

The five categories occupy different regions of the objective space of Eq.~\eqref{eq:fao}.

\paragraph{Communication.}
The reported payloads differ by roughly three orders of magnitude, from tens of megabytes per round for adapters to a few hundred kilobytes per run for skills, although the studies report them at different granularities.
The order of the categories is not fixed.
Accumulated text can grow large, and parameter updates can be compressed to a few kilobytes~\cite{Qin2024FedKSeed}.

\paragraph{Privacy.}
Weight sharing has the best understood attacks and defenses.
The methods that exchange text or rules~\cite{Wu2024FICAL,Yao2026FoT,Li2026FedAgentKE,Yang2026FederatedSkill,Hou2026FedWorld} rely on the fact that raw trajectories stay local, and none of them applies a formal privacy mechanism.
Two of them support the argument by an empirical audit with reconstruction attacks~\cite{Yao2026FoT,Yang2026FederatedSkill}.
The privacy problem thus shifts from the inversion of numerical updates to the disclosure of content in text, for which formal tools are less developed.

\paragraph{Access to the backbone.}
Weight sharing requires open weights and a common architecture.
The sharing of prompts, memory items, tool knowledge, rubrics, rules, and skills remains possible when every client uses a different closed model.

\paragraph{Maturity.}
Federated fine-tuning and federated knowledge graph embedding are established.
For rewards, methods exist for the alignment of language models but not yet for agent trajectories.
For memory, tool use, rules, and skills, the first federated methods have appeared, most of them as preprints.
Each of these methods optimizes performance for one component and treats privacy and communication as properties of its design.
To our knowledge, none formulates them as objectives to be traded off against performance, across components, as Eq.~\eqref{eq:fao} does.

\section{Knowledge Types in the Running Example}
\label{app:levels}

This appendix extends Section~\ref{sec:levels}.
We use one running example.
Several banks deploy agents for customer identity verification.
Each agent talks to applicants, calls tools for document recognition and registry lookup, and decides whether an application passes, fails, or is escalated to a human reviewer.
Customer records and conversation logs cannot leave the bank.

\subsection{The Four Types in the Running Example}
\label{app:levels_detail}

\paragraph{Experience.}
The sender shares what it has learned from successes and failures, stated as a pattern and not as a case.
In the running example, one bank observes that applications fail at the matching step when the name on the document and the name in the registry are transliterated differently, and that normalizing the transliteration before the lookup avoids the failure.
The shared artifact is this insight.
Methods that distill such insights from the trajectories of a single agent are well developed~\cite{Shinn2023Reflexion,Zhao2024ExpeL,Ouyang2026ReasoningBank}.
In a federation, a first distillation runs at the client, so that only summaries are uploaded, and a second one at the coordinator, which condenses the summaries of all clients into a shared library~\cite{Yao2026FoT,Li2026FedAgentKE}.

\paragraph{Skills.}
The sender shares a reusable procedure.
A skill for verifying a business license may consist of four steps: recognize the document, query the registry by license number, compare the legal representative, and escalate on mismatch.
The skill is parameterized, so that the receiver binds each step to its own tools.
Skills can be induced from trajectories as workflows or programs~\cite{Wang2024Voyager,Wang2025AWM,Zheng2025SkillWeaver}, and they can be exchanged as patches to a skill library~\cite{Yang2026FederatedSkill}.
A skill carries more structure than an insight and is more directly executable, but it also discloses more of the sender's business process.

\paragraph{Rewards.}
The sender shares what it has learned about how to judge a trajectory.
Clients hold the ground truth, such as the decisions of human reviewers, but each has few labels.
The coordinator may hold a stronger evaluator but has no access to real outcomes.
Federated training of the evaluator combines the two. It has been developed for preference models of language models~\cite{Wu2025FedBiscuit,Kim2026FedGD} and remains open for agent trajectories (Section~\ref{app:tax_reward}).
A shared rule in the running example is that a verification dialogue is non-compliant if the agent tells the applicant which field failed the check.
An improved evaluator benefits every later optimization step of the receiver, since it supplies the signal on which the policy, the memory, and the skills are tuned.

\paragraph{Ontologies.}
The sender shares how its concepts relate to those of others.
One bank records the \emph{beneficial owner} of a company, and another records the \emph{ultimate controlling person}; an alignment states whether the two are the same concept and under which conditions.
Language models and agents are increasingly used to find such correspondences~\cite{Qiang2024AgentOM}, and once concepts are aligned, their embeddings can be improved jointly without exchanging the underlying records~\cite{Chen2021FedE,Peng2021FKGE}.

\paragraph{Rules.}
The sender shares what it has learned about the effects of actions.
In the running example, one bank finds that three failed one-time codes lock the account of an applicant for 30 minutes, so that an agent should escalate instead of retrying.
FedWorld federates such knowledge as abstract transition rules and labels each rule as shared, specific to a cluster of clients, private, or unresolved~\cite{Hou2026FedWorld}.

\paragraph{Interactions among the types.}
The five types differ in function and are not ordered by abstraction.
Ontology alignment makes insights, skills, and rules easier to transfer, because these artifacts are written in terms of concepts that the receiver must ground, although natural language often aligns concepts implicitly.
The reward type acts at both ends of the loop: at the sender it selects which trajectories are worth distilling, and at the receiver it can estimate whether an adopted item helps when observed outcomes are scarce.
A federation without aligned concepts or a reliable evaluator finds it harder to tell useful knowledge from noise.

\subsection{From Shared Knowledge to Local Improvement}
\label{app:propagation}

The type of shared knowledge does not determine the component in which the receiver stores it.
An insight can be written to memory and retrieved when a similar task arrives, added to the system prompt, used as a training example for the policy, or used as a negative example for the evaluator.
Table~\ref{tab:propagation} lists the combinations.
The relation between the types of this section and the categories of Section~\ref{sec:taxonomy} is therefore many-to-many: the taxonomy describes what the federation maintains, and the types describe what travels between clients.

\begin{table}[!htb]
\centering
\footnotesize
\renewcommand{\arraystretch}{1.25}
\setlength{\tabcolsep}{4pt}
\begin{tabularx}{\textwidth}{@{}l *{5}{>{\raggedright\arraybackslash}X}@{}}
\toprule
\textbf{Shared type} & \textbf{Policy} $\pi_k$ & \textbf{Memory} $\Mem_k$ & \textbf{Tools} $\Tool_k$ & \textbf{Reward} $\Rew_k$ & \textbf{Structured knowledge} $\Onto_k$ \\
\midrule
Experience &
demonstrations for prompting or fine-tuning &
insights stored for retrieval &
corrected tool descriptions, known failure cases &
positive and negative examples for the evaluator &
new rules derived from recurring patterns \\
Skill &
procedures placed in the prompt &
procedures retrieved on demand &
compositions of tools &
checkpoints that a correct trajectory must pass &
extension of the skill library \\
Rule &
preconditions of actions stated in the prompt &
rules retrieved for similar states &
expected effects and failure modes of tool calls &
penalties for actions with known bad effects &
extension of the rule base \\
Reward &
training signal for policy optimization &
filter for what is written to memory &
reliability scores of tools &
parameters or rubrics of the evaluator &
validation of new rules and skills \\
Ontology &
shared vocabulary in prompts &
retrieval keys over aligned concepts &
mapping of tool arguments to shared concepts &
criteria stated over shared concepts &
aligned concepts and relations \\
\bottomrule
\end{tabularx}
\caption{How knowledge of each type (rows) can improve each component of the receiving agent (columns). One shared artifact can improve several components.}
\label{tab:propagation}
\end{table}

\section{Extensions of the Formulation}
\label{app:extensions}

Three further concerns fit into the same formulation.
Inference cost (tokens, latency, and fees for tool calls) can be charged inside $u_k$ or added as a fourth objective.
Non-stationarity can be modeled by letting the task distribution $\Dist_k^t$ depend on the round, in which case utility is evaluated on the current distribution and the agent must adapt without losing capabilities it still needs.

Protocols without a coordinator, such as the exchange of embeddings between pairs of clients in FKGE~\cite{Peng2021FKGE}, compute $\Agg$ among the clients.
They are covered if the communication cost counts every message $\mathrm{msg}$ in the set $\mathcal{C}_t(\Phi)$ of messages that clients exchange in round $t$,
\begin{equation}
\label{eq:comm_general}
\Comm(\Phi) = \sum_{t=0}^{T-1} \sum_{\mathrm{msg} \in \mathcal{C}_t(\Phi)} |\mathrm{msg}|,
\end{equation}
which reduces to Eq.~\eqref{eq:objectives} for the star topology with a coordinator that the main text considers.

\section{Proofs}
\label{app:proofs}

\begin{proposition}[Scalarized forms]
\label{prop:inclusion}
Assume that the maximizers of Eqs.~\eqref{eq:fao_weighted} and~\eqref{eq:fao_constrained} exist.
Every maximizer of Eq.~\eqref{eq:fao_weighted} with $\lambda_p, \lambda_c > 0$ is Pareto optimal, and every Pareto-optimal protocol $\Phi^{\star}$ maximizes Eq.~\eqref{eq:fao_constrained} with the budgets $\bar\epsilon_k = \Leak_k(\Phi^{\star})$ and $B = \Comm(\Phi^{\star})$.
\end{proposition}

\begin{proof}
Eq.~\eqref{eq:fao_weighted} maximizes $w(\Phi) = F_1(\Phi) + \lambda_p F_2(\Phi) + \lambda_c F_3(\Phi)$.
If a maximizer $\Phi$ were dominated by some $\Phi'$, then $F_i(\Phi') \ge F_i(\Phi)$ for all $i$ with at least one strict inequality, and since all weights are positive, $w(\Phi') > w(\Phi)$, a contradiction.
For the second claim, $\Phi^{\star}$ is feasible for the stated budgets.
If a feasible $\Phi'$ had $F_1(\Phi') > F_1(\Phi^{\star})$, then $\Leak_k(\Phi') \le \Leak_k(\Phi^{\star})$ for all $k$ would give $F_2(\Phi') \ge F_2(\Phi^{\star})$, and $\Comm(\Phi') \le \Comm(\Phi^{\star})$ would give $F_3(\Phi') \ge F_3(\Phi^{\star})$, so $\Phi'$ would dominate $\Phi^{\star}$, a contradiction.
\end{proof}

The constrained form bounds the leakage of every client, which is stronger than a bound on the average $F_2$, and it still reaches every Pareto-optimal protocol.
The converse inclusion fails in general, since a maximizer of Eq.~\eqref{eq:fao_constrained} may be only weakly Pareto optimal~\cite{Miettinen1999NMOO}.

\begin{proposition}[Gated adaptation]
\label{prop:gate}
Let the items $z_1, \dots, z_m$ be fixed before validation and let $\oplus$ be deterministic.
Let $\hat U_k(A)$, computed once for each agent $A$, be the mean of $u_k$ over $n = n_k^{\mathrm{val}}$ tasks drawn independently from $\Dist_k$, independent of the items and of $A_k^{\mathrm{loc}}$, with one trajectory sampled from $p_A$ for each task.
With $e_k$ as in Eq.~\eqref{eq:gate_bound}, the gate of Eq.~\eqref{eq:gate} yields $\mathrm{CG}_k \ge 0$ with probability at least $1 - \eta$.
\end{proposition}

\begin{proof}
The gate processes the items in order, so it can only reach agents of the form $A_k^{\mathrm{loc}} \oplus z_{i_1} \oplus \dots \oplus z_{i_s}$ with $i_1 < \dots < i_s$, one for each subset of the items.
Let $\mathcal{C}$ be this set of at most $2^m$ agents.
For a fixed $A \in \mathcal{C}$, $\hat U_k(A)$ is the mean of $n$ independent variables in $[0,1]$ with mean $U_k(A)$, so Hoeffding's inequality gives
\begin{equation*}
\Pr\big[\,|\hat U_k(A) - U_k(A)| > e_k\big] \ \le\ 2\exp\!\big(-2 n e_k^2\big) \ =\ 2 \cdot 2^{-(m+1)}\, \eta \ =\ 2^{-m}\, \eta .
\end{equation*}
A union bound over $\mathcal{C}$ shows that the event $E = \{|\hat U_k(A) - U_k(A)| \le e_k \ \forall A \in \mathcal{C}\}$ has probability at least $1 - \eta$.
On $E$, every accepted step satisfies
\begin{equation*}
U_k\big(A_k^{(i)}\big) - U_k\big(A_k^{(i-1)}\big) \ \ge\ \hat U_k\big(A_k^{(i)}\big) - \hat U_k\big(A_k^{(i-1)}\big) - 2 e_k \ \ge\ 0,
\end{equation*}
and every rejected step leaves the agent unchanged.
Summing over the steps gives $\mathrm{CG}_k = U_k(A_k^{+}) - U_k(A_k^{\mathrm{loc}}) = \sum_{i=1}^{m} \big[U_k(A_k^{(i)}) - U_k(A_k^{(i-1)})\big] \ge 0$ on $E$.
\end{proof}

\end{document}